\documentclass[letterpaper, 10 pt, conference]{ieeeconf}  

\IEEEoverridecommandlockouts                              

\usepackage{graphics}
\usepackage{epsfig}
\usepackage{mathptmx}
\usepackage{times}
\usepackage{amsmath}
\usepackage{amssymb}

\usepackage{amsthm}
\let\labelindent\relax
\usepackage{enumitem}

\usepackage{tabularx}
\usepackage{algorithm}
\usepackage{algpseudocode}
\usepackage{array}
\usepackage{booktabs}
\usepackage{threeparttable}
\usepackage{siunitx}
\usepackage{subfig}
\newtheorem{proposition}{Proposition}
\theoremstyle{remark} 
\newtheorem{remark}{Remark}

\title{\LARGE \bf
L-Learning : A Lyapunov-Based Approach Leveraging Lagrangian Mechanics for Efficient and Stable Robot Tracking
}

\author{
Quan Quan$^{1}$ and Hao Li$^{1}$
\thanks{$^{1}$Quan Quan and Hao Li are with School of Automation Science and Electrical Engineering, Beihang University, Beijing 100191, China
        {\tt\small qq\_buaa@buaa.edu.cn}
        {\tt\small lh912@buaa.edu.cn}}%
}

\begin{document}

\maketitle
\thispagestyle{empty}
\pagestyle{empty}

\begin{abstract}
This paper presents L-Learning, a novel data-driven control framework for robotics that integrates Lyapunov stability theory with Lagrangian mechanics to enhance trajectory tracking performance. While traditional control methods often suffer from performance degradation in dynamic and uncertain environments, data-driven approaches—though more adaptable—are frequently limited by high sample complexity and a lack of rigorous stability guarantees. L-Learning mitigates these challenges by explicitly learning the system's energy function from data, thereby optimizing performance while ensuring closed-loop stability intrinsically. Characterized by superior control accuracy, theoretical stability guarantees, and high sample efficiency, L-Learning represents a promising solution for practical robotic applications.
\end{abstract}

\section{INTRODUCTION}

The prominence of data-driven control in robotic systems has increased significantly in recent years, driven by the growing complexity of operational tasks. Conventional control methodologies, which rely on explicit mathematical models of system dynamics, often prove insufficient in environments where such models are elusive or subject to high uncertainty. Conversely, data-driven approaches exploit available sensory data to synthesize optimal control policies, offering superior adaptability and scalability in real-world scenarios. Consequently, in the domain of robotics, which is characterized by frequent interactions with stochastic and dynamic environments, the capacity to implement learning-based controllers that simultaneously ensure robust performance and rigorous system stability is paramount.

Currently, the landscape of robot learning control is dominated by two primary paradigms: reinforcement learning (RL) \cite{sutton1998reinforcement,bertsekas2019reinforcement} and certificate learning \cite{boffi2021learning,dawson2023safe}. RL strategies, including Q-Learning \cite{watkins1992q} and deep reinforcement learning \cite{mnih2013playing,schulman2015trust,haarnoja2018soft}, focus on policy optimization via exploration and exploitation, often yielding high performance in complex tasks. However, these methods frequently exhibit prohibitive sample complexity and a lack of theoretical stability guarantees. These deficiencies are particularly critical in safety-critical applications. Although enhancements such as experience replay, target networks \cite{heess2015memory}, asynchronous methods \cite{mnih2016asynchronous}, and Proximal Policy Optimization (PPO) \cite{schulman2017proximal} have been proposed to mitigate these issues, stability concerns persist. In contrast, certificate learning has gained traction for its ability to provide formal safety assurances. By incorporating certificates such as Lyapunov or barrier functions, these methods ensure that learned controllers adhere to strict stability and safety constraints. Within this framework, Lyapunov function learning has emerged as a cornerstone for guaranteeing stability in dynamic systems \cite{dawson2023safe,yang2024lyapunov,chang2019neural,dai2021lyapunov,zhou2022neural}. By leveraging Lyapunov functions, defined as mathematical constructs that verify system stability through energy dissipation, data-driven controllers can be engineered to optimize performance while strictly maintaining stability. Notably, D-Learning \cite{quan2024control} facilitates the direct learning of Lyapunov functions from data to derive stable control policies, extending the utility of Q-Learning-like approaches to complex nonlinear systems.

\begin{figure}[h]
    \centering
    \includegraphics[width=1.0\linewidth]{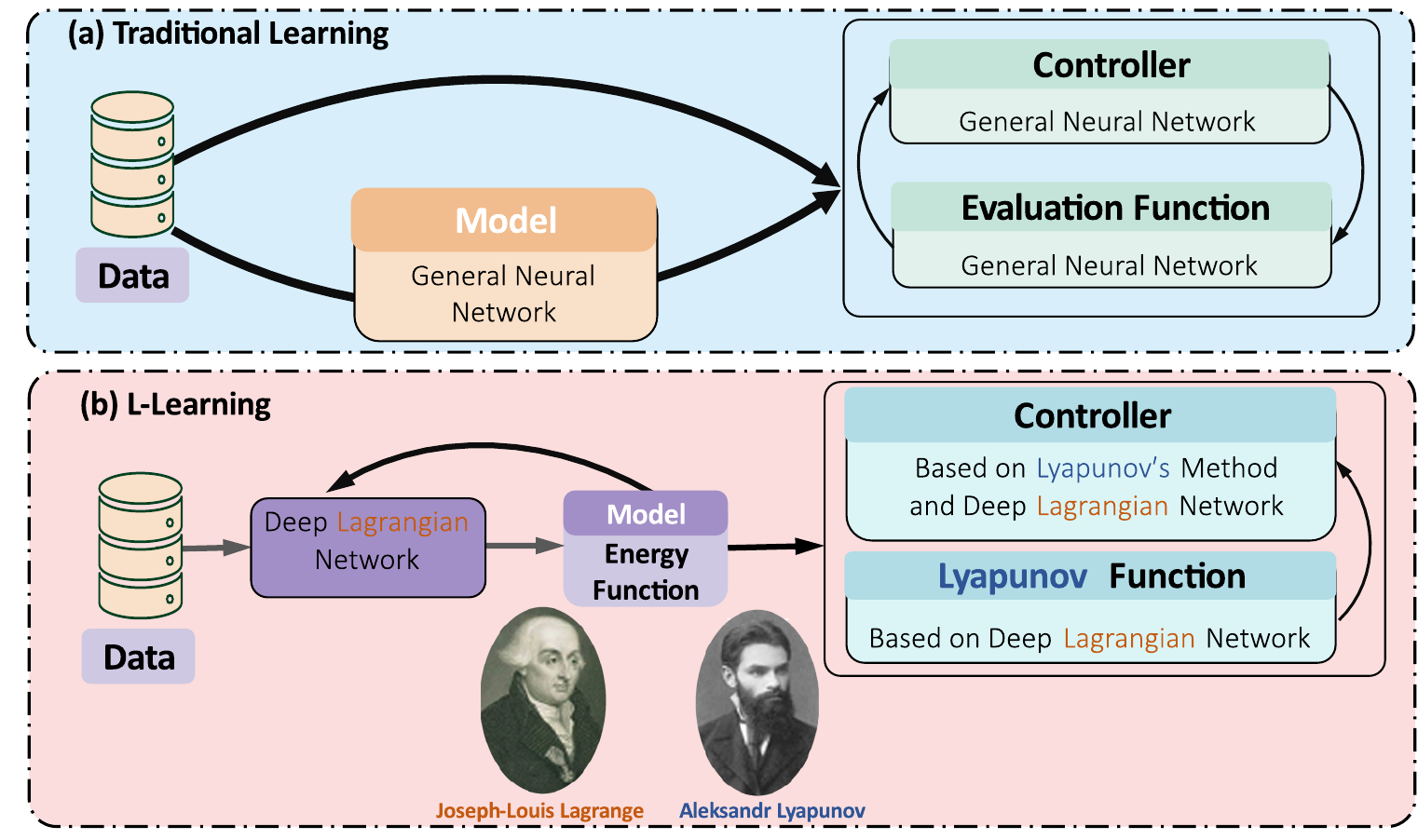}
    \caption{Comparison of Learning-based Control Flow: (a) Traditional learning algorithm vs. (b) L-Learning.}
    \label{fig:method_compare}
\end{figure}

As depicted in Figure \ref{fig:method_compare}(a), standard learning-based control frameworks rely on the iterative and joint optimization of multiple networks to refine the system model, controller, and value function. While powerful, these architectures suffer from inherent scalability constraints. Crucially, their high sample complexity necessitates prohibitive volumes of training data and computational power. Furthermore, the process of learning evaluation functions (e.g., Lyapunov candidates for stability certification) is resource-intensive, limiting their deployment in complex, high-dimensional robotic environments. This underscores the urgency for advanced control paradigms capable of improving data efficiency and reducing the computational costs associated with training.

In response to these challenges, physics-informed machine learning \cite{karniadakis2021physics,cai2021physics,cuomo2022scientific} has gained traction by incorporating physical inductive biases into learning algorithms. Notable approaches \cite{cranmer2020lagrangian,lutter2023combining, greydanus2019hamiltonian} explicitly integrate Lagrangian or Hamiltonian mechanics into the deep learning architecture. This strategy constrains the hypothesis space to adhere to underlying physical laws, allowing models to synergize the expressiveness of neural networks with rigorous dynamic consistency, thereby achieving enhanced prediction accuracy and generalization.

Predicated on this philosophy, we introduce L-Learning, an innovative paradigm that integrates Lagrangian mechanics and Lyapunov stability theory for robust robot trajectory tracking. In Lagrangian mechanics \cite{goldstein2002classical,morin2008introduction}, the system's governing dynamics are fully characterized by the Lagrangian $L$, a quantity derived from the system's kinetic and potential energy. In parallel, Lyapunov theory provides a rigorous tool for stability analysis, employing an energy-like scalar function whose temporal evolution dictates the system's convergence properties. Specifically, a strictly decreasing Lyapunov function corresponds to energy dissipation \cite{khalil2002nonlinear}, ensuring asymptotic stability analogous to physical damping mechanisms. Inspired by this insight, L-Learning aims to directly learn the system's energy function and reconstruct its dynamic characteristics from observational data, as illustrated in Figure \ref{fig:method_compare}(b). By embedding the learned dynamics into both stability certification and control design, the framework constructs a stabilizing controller concurrently with the approximation of system dynamics. This ensures not only the alignment of the learned energy model with the physical system but also that the controller continuously enhances tracking precision subject to guaranteed system stability.

Our proposed framework offers the following distinct advantages:
\begin{itemize}[leftmargin=*,label=$\bullet$]
\item {\textbf{Learning Efficiency}:} Unlike conventional architectures that necessitate separate neural networks for policy execution and value estimation, our method obviates this redundancy by directly learning the system dynamics within a unified network, from which both the Lyapunov function and the controller are derived. This integration renders the learning process significantly more streamlined and computationally efficient. Furthermore, by embedding Lagrangian mechanics as a physical prior and constraining the optimization landscape via the Lagrange equation, the method substantially reduces sample complexity and enhances data utilization.

\item {\textbf{Stable and Accurate Tracking}:} The proposed approach leverages data-driven dynamics to synthesize a trajectory tracking controller governed by Lyapunov stability criteria. Crucially, this design offers a dual benefit: first, it strictly enforces the negative definiteness of the Lyapunov derivative, thereby providing a theoretical guarantee of asymptotic tracking error convergence. Second, by explicitly incorporating the learned system dynamics, the controller effectively compensates for complex nonlinear forces, ensuring high-fidelity tracking performance.
\end{itemize}

\section{PRELIMINARIES}

\subsection{The Lagrangian Mechanics}
As established in classical mechanics literature \cite[pp.~16--22]{goldstein2002classical} \cite[pp.~218--281]{morin2008introduction}, the fundamental quantity in this framework is the Lagrangian, denoted as $L$. Defined as the difference between the system's kinetic energy and potential energy (i.e., $L=K-P$), the Lagrangian characterizes the system's energy landscape. From this scalar function, the Lagrange equations can be derived as:
\begin{equation}
\frac{\mathrm{d}}{\mathrm{d}t}\frac{\partial L(\mathbf{q},\dot{\mathbf{q}})}{\partial \dot{\mathbf{q}}} - \frac{\partial L(\mathbf{q},\dot{\mathbf{q}})}{\partial \mathbf{q}} = \mathbf{u}
\label{eq_Lagran}
\end{equation}
where:
\begin{itemize}
    \item $\mathbf{q} = \left[ q_1 \ldots q_n \right]^\top \in \mathbb{R}^n$ is the generalized position and $\dot{\mathbf{q}} = \left[ \dot{q}_1 \ldots \dot{q}_n \right]^\top \in \mathbb{R}^n$ is the generalized velocity; 
    \item $\mathbf{u} = \left[ u_1 \ldots u_n \right]^\top \in \mathbb{R}^n$ is the non-conservative generalized force acting on the system.
\end{itemize}

The scalar form of the Lagrange equation (\ref{eq_Lagran}) is given by:
\begin{equation} 
\frac{\mathrm{d}}{\mathrm{d}t}\frac{\partial L(\mathbf{q},\dot{\mathbf{q}})}{\partial \dot{q}_i} - \frac{\partial L(\mathbf{q},\dot{\mathbf{q}})}{\partial q_i} = u_i, \quad i = 1, \dots, n.
\label{eq_Lagran_scalar}
\end{equation}
Defining $\mathbf{D}(\mathbf{q}) \in \mathbb{R}^{n \times n}$ as the inertia matrix, the Lagrangian can be expressed as:
\begin{equation}
    L(\mathbf{q},\dot{\mathbf{q}}) = \frac{1}{2} \dot{\mathbf{q}}^\top \mathbf{D}(\mathbf{q}) \dot{\mathbf{q}} - P(\mathbf{q}).
    \label{eq_Lagran_K_P}
\end{equation}
Substituting (\ref{eq_Lagran_K_P}) into (\ref{eq_Lagran}) yields the standard robotic dynamics equation:
\begin{equation}
\mathbf{D}(\mathbf{q})\ddot{\mathbf{q}} + \mathbf{C}(\mathbf{q},\dot{\mathbf{q}})\dot{\mathbf{q}} + \mathbf{G}(\mathbf{q}) = \mathbf{u},
\label{eq_universal}
\end{equation}
where $\mathbf{C}(\mathbf{q},\dot{\mathbf{q}}) \in \mathbb{R}^{n \times n}$ denotes the Coriolis and centrifugal matrix, and $\mathbf{G}(\mathbf{q}) \in \mathbb{R}^n$ is the gravitational vector. Consequently, the dynamics matrices $\mathbf{D}$, $\mathbf{C}$, and $\mathbf{G}$ (with dependencies omitted for notational convenience) are fully derived from, and characterized by, the Lagrangian $L(\mathbf{q},\dot{\mathbf{q}})$.

\begin{figure*}[thbp]
\vspace{10pt}
\centering
\subfloat{\includegraphics[width=6.5in]{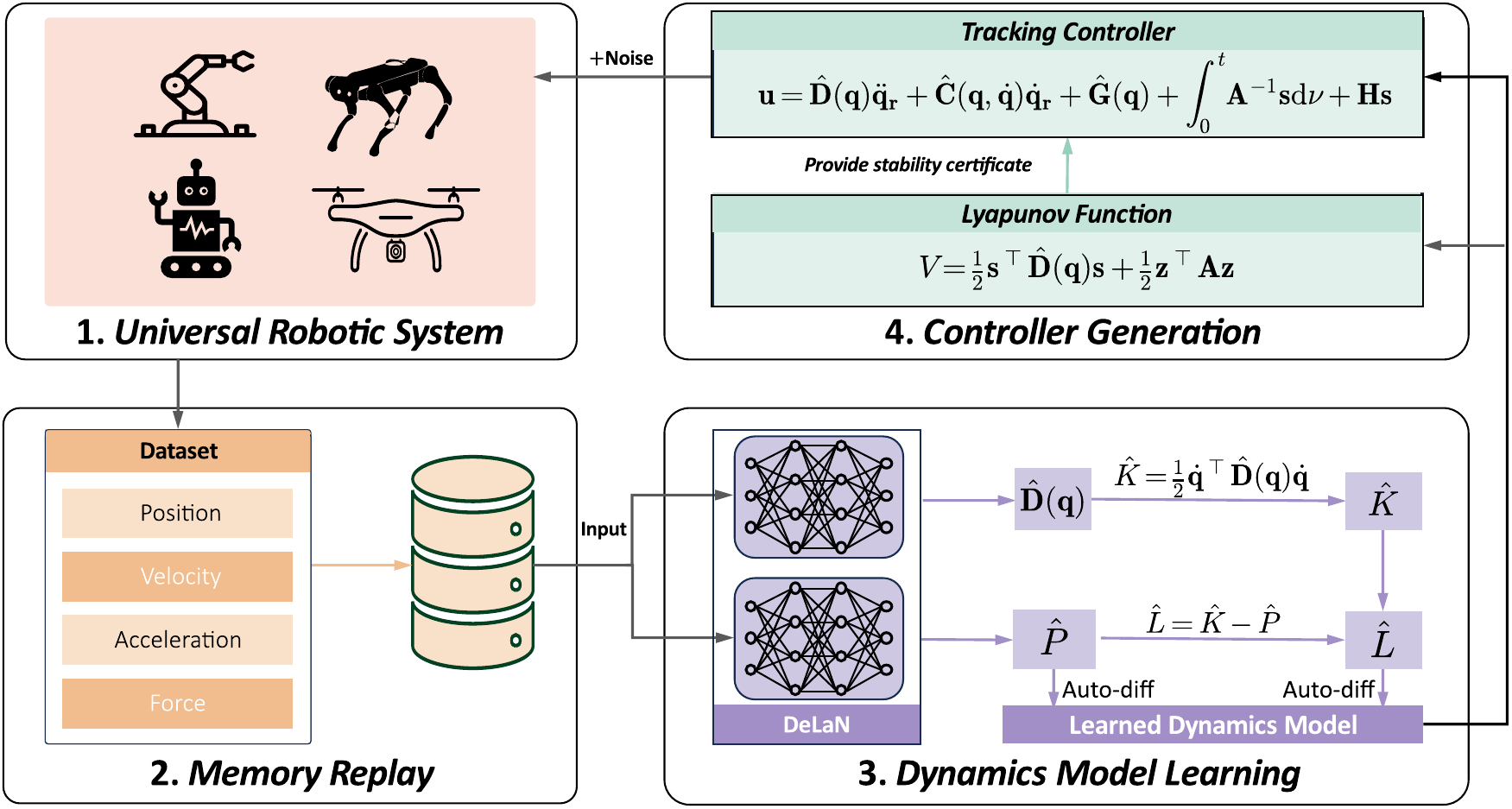} \label{X1}}
\caption{\textrm{The flowchart of the L-Learning method}. The figure illustrates the general approach to tracking control based on L-Learning, where the system provides training data to the network, the network continuously optimizes its estimation of the system's dynamics model based on the input data and then generates a controller to control the system using the estimated dynamics model.} 
\label{figure:L-Learning}
\end{figure*}

\subsection{Tracking Problem Formulation}
Consider an $n$-degree-of-freedom system with generalized coordinates $\mathbf{q}(t)$ and a desired trajectory $\mathbf{q}_\mathrm{d}(t)$. The trajectory tracking problem entails designing a controller that ensures $\mathbf{q}(t)$ converges to $\mathbf{q}_\mathrm{d}(t)$. Traditional nonlinear methods, like backstepping control \cite{krstic1995nonlinear}, necessitate explicit knowledge of the system dynamics to ensure stability. To address scenarios where specific system model parameters are unknown, this paper adopts a data-driven methodology to learn the system dynamics. Upon obtaining the learned model, the system's dynamic equation (\ref{eq_universal}) can be reformulated as
\begin{align}
&\mathbf{\hat D}(\mathbf{q})\mathbf{\ddot q} + \mathbf{\hat C}(\mathbf{q},\mathbf{\dot q})\mathbf{\dot q} + \mathbf{\hat G}(\mathbf{q}) = \mathbf{u} - \mathbf{d}, \label{eq_learn_dyn}\\
&\mathbf{d} =\mathbf{\tilde{D}}(\mathbf{q})\mathbf{\ddot q}+\mathbf{\tilde{C}}(\mathbf{q},\mathbf{\dot q})\mathbf{\dot q}+\mathbf{\tilde{G}}(\mathbf{q}) \label{eq_d}
\end{align}
where:
\begin{itemize}
    \item $\mathbf{\hat D}(\mathbf{q})$, $\mathbf{\hat C}(\mathbf{q},\mathbf{\dot q})$ and $\mathbf{\hat G}(\mathbf{q})$ are the learned dynamics matrices (i.e., the dynamics model);
    \item $\mathbf{\tilde D}(\mathbf{q})$, $\mathbf{\tilde C}(\mathbf{q},\mathbf{\dot q})$ and $\mathbf{\tilde G}(\mathbf{q})$ are the estimated errors of the corresponding matrices. The actual matrix can be described as the sum of the learned value and the error, i.e., $\mathbf{D}(\mathbf{q}) = \mathbf{\hat{D}}(\mathbf{q}) + \mathbf{\tilde{D}}(\mathbf{q}),\mathbf{C}(\mathbf{q},\mathbf{\dot q}) = \mathbf{\hat{C}}(\mathbf{q},\mathbf{\dot q}) + \mathbf{\tilde{C}}(\mathbf{q},\mathbf{\dot q}),\mathbf{G}(\mathbf{q}) = \mathbf{\hat{G}}(\mathbf{q}) + \mathbf{\tilde{G}}(\mathbf{q})$;
    \item $\mathbf{d}$ is the total estimation error of the dynamics model.
\end{itemize}
Overall, the primary objective of this paper is to learn the system's dynamics model, comprising the estimated terms $\mathbf{\hat{D}}$, $\mathbf{\hat{C}}$, and $\mathbf{\hat{G}}$. Based on these learned dynamics, we rigorously prove Lyapunov stability guarantees and synthesize a controller to achieve precise tracking of the desired trajectory. Crucially, all these contributions are achieved without the prior knowledge of the system's specific dynamics.

\section{PROPOSED METHOD: L-LEARNING}

\subsection{Outline} 
Figure \ref{figure:L-Learning} illustrates the overall architecture of the proposed L-Learning framework, which comprises four tightly interconnected components: (1) \textit{the robotic system}, which generates state and action data through interaction with the physical environment; (2) \textit{the memory replay module}, which stores historical data to enhance the diversity of samples within the memory storage; (3) \textit{the dynamics learning phase}, which utilizes Deep Lagrangian Networks (DeLaN) \cite{lutter2023combining} to learn the system's energy function and leverages automatic differentiation to obtain the corresponding dynamics matrices estimates; and (4) \textit{the controller synthesis phase}, which leverages the learned model to design a stable tracking controller. Unlike certificate learning methods that typically rely on pre-collected and stable datasets to construct Lyapunov certificates, our approach directly learns system dynamics from raw interaction data, subsequently allowing for the construction of a Lyapunov function. By unifying dynamics learning, stability certification, and control strategy optimization, this framework achieves a seamless data-driven integration. The cornerstone of this method is the synergy between Deep \textit{L}agrangian Networks (to enhance physical consistency) and \textit{L}yapunov stability analysis (to provide rigorous closed-loop guarantees). This dual foundation not only improves security and learning efficiency but also justifies the nomenclature "\textit{L}-Learning," highlighting its roots in Lagrangian mechanics and Lyapunov stability theory.

\subsection{Dynamics Model Learning}
Dynamics model learning is performed using Deep Lagrangian Networks (DeLaN), a framework that embeds Lagrangian mechanics directly into deep neural network architectures. Formally, we denote the parameterized model as $\hat{\mathbf{D}} = \hat{\mathbf{D}}_\vartheta(\mathbf{q}) = \hat{\mathbf{R}}_\vartheta\hat{\mathbf{R}}_\vartheta^\top + \tau \mathbf{I}$, where $\hat{\mathbf{R}}_\vartheta$ denotes a lower triangular matrix with positive diagonals, $\tau > 0$ is a small constant ensuring positive definiteness, and the neural network $\hat{\mathbf{D}}_\vartheta$, parameterized by weights $\vartheta$, maps the state input $\mathbf{q}$ to the estimated inertia matrix. The core objective of DeLaN is to approximate the true Lagrangian of the system. By explicitly parameterizing both the inertia matrix and the potential energy via dedicated networks, the learned Lagrangian can be expressed as:
\begin{equation}
    \hat{L} = \hat{L}_{\vartheta,\xi}(\mathbf{q},\dot{\mathbf{q}}) = \frac{1}{2}\dot{\mathbf{q}}^\top\hat{\mathbf{D}}_\vartheta(\mathbf{q})\dot{\mathbf{q}} - \hat{P}_{\xi}(\mathbf{q}).
    \label{net_lagrangian}
\end{equation}
Upon approximating the system's Lagrangian, the associated non-conservative forces can be computed using an inverse model derived from the Euler-Lagrange equations (\ref{eq_Lagran}). The inverse dynamics are governed by:
\begin{align}
    \hat{\mathbf{u}} 
    &= \frac{\mathrm{d}}{\mathrm{d}t}\frac{\partial \hat{L}}{\partial \dot{\mathbf{q}}} - \frac{\partial \hat{L}}{\partial \mathbf{q}} \notag\\
    &= \frac{\partial^2 \hat{L}}{\partial \dot{\mathbf{q}}^2}\ddot{\mathbf{q}} + \frac{\partial^2 \hat{L}}{\partial \mathbf{q}\partial \dot{\mathbf{q}}}\dot{\mathbf{q}} - \frac{\partial \hat{L}}{\partial \mathbf{q}}.
\label{eq_inverse_model}
\end{align}
Given that the system's internal energy is not directly measurable, we compute the corresponding non-conservative forces $\hat{\mathbf{u}}$ from the approximated Lagrangian $\hat{L}$ via the inverse dynamics formulation in Equation (\ref{eq_inverse_model}). Consequently, the network training is supervised by leveraging the empirical dynamic transitions $(\mathbf{q},\dot{\mathbf{q}},\ddot{\mathbf{q}},\mathbf{u})$ in conjunction with this inverse model, where $\mathbf{q}$, $\dot{\mathbf{q}}$, and $\ddot{\mathbf{q}}$ denote the system's generalized coordinates, velocities, and accelerations, respectively, and $\mathbf{u}$ represents the applied control inputs. The optimization loss function is formulated as:
\begin{equation}
    \mathrm{Loss} = \frac{1}{N} \sum_{i=1}^{N} \lVert \hat{\mathbf{u}}^{(i)} - \mathbf{u}^{(i)} \rVert
    \label{eq_loss}
\end{equation}
where $\hat{\mathbf{u}} = \frac{\partial^2 \hat{L}}{\partial \dot{\mathbf{q}}^2}\ddot{\mathbf{q}} + \frac{\partial^2 \hat{L}}{\partial \mathbf{q}\partial \dot{\mathbf{q}}}\dot{\mathbf{q}} - \frac{\partial \hat{L}}{\partial \mathbf{q}}$. Therefore, the optimization problem for the dynamics model learning can be expressed as:
\begin{equation}
(\vartheta^*, \xi^*) = \arg\min_{\vartheta, \xi} \frac{1}{N} \sum_{i=1}^{N} \lVert \hat{\mathbf{u}}^{(i)} - \mathbf{u}^{(i)} \rVert.
\label{eq_opti}
\end{equation}

As the proposed learning process is developed utilizing the JAX library \cite{bradbury2021jax}, we exploit its inherent automatic differentiation (auto-diff) capabilities to compute the gradients of the neural network outputs. This mechanism natively supports the requisite partial differentiation of the learned Lagrangian $\hat{L}$ with respect to both $\mathbf{q}$ and $\mathbf{\dot q}$. Consequently, solving the optimization problem defined in Equation (\ref{eq_opti}) yields a highly precise model of the Lagrangian $\hat{L}$. Furthermore, to explicitly reconstruct the underlying system dynamics (namely, the matrices $\hat{\mathbf{D}}$ and $\hat{\mathbf{C}}$, alongside the vector $\hat{\mathbf{G}}$), we formalize the parameter extraction process. Grounded in classical robotics principles \cite{siciliano2009robotics} and the auto-diff mechanism, \textit{Proposition 1} is formulated as:

\begin{proposition}
If the learned Lagrangian $\hat{L}(\mathbf{q},\dot{\mathbf{q}})$ is parametrized as (\ref{net_lagrangian}) with a set of minimal and independent generalized coordinates $\mathbf{q}\in \mathbb{R}^n$, the estimated value of each component of the dynamics $\hat{\mathbf{D}}(\mathbf{q})\in \mathbb{R}^{n\times n}$, $\hat{\mathbf{C}}(\mathbf{q},\dot{\mathbf{q}})\in \mathbb{R}^{n\times n}$ and $\hat{\mathbf{G}}(\mathbf{q})\in \mathbb{R}^{n}$ can be obtained by the following formulas:
\begin{align}
     \hat{D}_{ij}&=\frac{\partial^2\hat{L}}{\partial\dot{q}_i\partial\dot{q}_j} \label{L_D}\\
     \hat{C}_{ij}&=\sum_{k=1}^n\frac{1}{2}\left(\frac{\partial^3\hat{L}}{\partial q_k\partial\dot{q}_i\partial\dot{q}_j}+\frac{\partial^3\hat{L}}{\partial q_j\partial\dot{q}_i\partial\dot{q}_k}-\frac{\partial^3\hat{L}}{\partial q_i\partial\dot{q}_j\partial\dot{q}_k}\right)\dot{q}_k \label{L_C}\\
     \hat{G}_j & = -\frac{\partial}{\partial q_{j}}\left(\hat{L}-\frac{1}{2}\sum_{i=1}^{n}\dot{q}_{i}\frac{\partial\hat{L}}{\partial\dot{q}_{i}}\right) .
     \label{L_G}
\end{align}
And the learned dynamics also conforms to the underlying physical properties of the robotic system, such as $\mathbf{\hat{D}}$ is a symmetric positive definite matrix and $\mathbf{\dot{\hat{D}}}-2\mathbf{\hat{C}}$ is a skew-symmetric matrix.
\label{pro_1}
\end{proposition}

\begin{proof}
The proof is analogous to the one for related properties in robot dynamics modeling. See \cite{spong2006robot}. 
\end{proof}

\begin{remark}
During the training phase, the estimated Lagrangian $\hat{L}$ converges toward the true system Lagrangian $L$, despite potential stochastic fluctuations in the optimization process. Consequently, the discrepancies between the true dynamics matrices and those derived via Equations (\ref{L_D}), (\ref{L_C}), and (\ref{L_G}) will systematically diminish. Acknowledging the inherent limits of data-driven approximation, these errors are ultimately bounded rather than strictly zero, satisfying $\|\hat{\mathbf{D}}-\mathbf{D}\|_F \rightarrow \epsilon_D$, $\|\hat{\mathbf{C}}-\mathbf{C}\|_F \rightarrow \epsilon_C$, and $\|\hat{\mathbf{G}}-\mathbf{G}\|_F \rightarrow \epsilon_G$, where $\epsilon_{(\cdot)}$ are small positive constants and $\lVert \cdot \rVert_{F}: \mathbb{R}^{m\times n}\rightarrow\mathbb{R}$ is the Frobenius norm.
\end{remark}

\subsection{Controller Generation based on Lyapunov Stability}
The primary objective of the proposed controller is to synthesize a control law  that ensures the system state $(\mathbf{q},\dot{\mathbf{q}})$  asymptotically tracks the reference trajectory $(\mathbf{q}_\mathrm{d},\dot{\mathbf{q}}_\mathrm{d})$ with robust stability. To facilitate the control design, we define the reference velocity ${\mathbf{\dot q}}_\mathrm{r}$ and the composite sliding surface variable $\mathbf{s}$ as follows:
\begin{equation}
    {\mathbf{\dot q}}_\mathrm{r} = {\mathbf{\dot q}}_\mathrm{d} + \boldsymbol{\Lambda} \mathbf{\tilde q}, \quad\mathbf{s}=\mathbf{\dot {\tilde q}} + \boldsymbol{\Lambda} \mathbf{\tilde q}={\mathbf{\dot q}}_\mathrm{r}-\dot{\mathbf{q}}
    \label{track_notations}
\end{equation}
where $\mathbf{\tilde q}=\mathbf{q}_\mathrm{d}-\mathbf{q}$ is defined as the position tracking error and $\boldsymbol{\Lambda}$ is a positive definite gain matrix, i.e., $\boldsymbol{\Lambda}\succ0$. 

Building upon Proposition~\ref{pro_1}, this study models the system dynamics based on the learned Lagrangian function. Given that the learned model cannot fully characterize the system's actual physical properties, modeling uncertainties and approximation errors are inherent. Consequently, the discrepancy between the learned and true dynamics is formulated as a bounded perturbation term $\mathbf{d}$. To ensure system stability during the learning process, a corresponding compensation mechanism is incorporated,
\begin{equation}
    \hat{\mathbf{d}} = \int {{{\bf{A}}^{ - 1}}{\bf{s}}{\rm{d}}t}, \quad \mathbf{z} \triangleq \hat{\mathbf{d}} - \mathbf{d}
    \label{eq_hat_d}
\end{equation}
where $\hat{\mathbf{d}}$ is an integral term to compensate for the dynamics learning error with matrix $\mathbf{A}\succ0$ and $\mathbf{z}$ is the
integral compensation error.

To analyze tracking convergence and the validity of the integral compensation, we define the Lyapunov function candidate:
\begin{equation}
V = \frac{1}{2}{\mathbf{s}^\top}\mathbf{\hat D}(\mathbf{q})\mathbf{s} + \frac{1}{2}{\mathbf{z}^\top}\mathbf{A}\mathbf{z}
\label{eq_Lya}
\end{equation}
where $\hat{\mathbf{D}}(\mathbf{q})$ is the learned inertia matrix. Because $\hat{\mathbf{D}}(\mathbf{q})$ is parameterized via Cholesky decomposition and $\mathbf{A} \succ 0$, $V$ is strictly positive definite. Mathematically, $\mathbf{s}^{\top}\hat{\mathbf{D}}(\mathbf{q})\mathbf{s} \geq \lambda_{\min}(\hat{\mathbf{D}}(\mathbf{q}))\|\mathbf{s}\|^{2}$ and $\mathbf{z}^\top\mathbf{A}\mathbf{z} \geq \lambda_{\min}(\mathbf{A})\|\mathbf{z}\|^2$, where $\lambda_{\min}(\cdot)$ denotes the minimum eigenvalue. By virtue of its structural parameterization, the learned inertia matrix $\hat{\mathbf{D}}(\mathbf{q})$ is uniformly positive definite; hence, there exists a strictly positive constant $\lambda_D > 0$ such that $\mathbf{s}^{\top}\hat{\mathbf{D}}(\mathbf{q})\mathbf{s} \geq \lambda_D\|\mathbf{s}\|^{2}$ for all $\mathbf{q}$. Therefore, $V \to \infty$ as $\|(\mathbf{s}, \mathbf{z})\| \to \infty$, verifying that the candidate function is strictly positive definite and radially unbounded, which is essential for global stability analysis. 

Based on the analysis above, the candidate Lyapunov function (\ref{eq_Lya}) can be used to guide the controller design. We prove \textit{Proposition 2}, which presents the controller design method to ensure the stability of tracking and error compensation.
\begin{proposition}
\textit{For the dynamics system (\ref{eq_learn_dyn}), suppose that $\dot{\mathbf{d}}$ is approximately equal to $\mathbf{0}$ and $\mathbf{H}\succ0$ , if the control law $\mathbf{u}$ is designed as:}
\begin{equation}
\mathbf{u}  = {\mathbf{\hat D}(\mathbf{q}){{\mathbf{\ddot q}}}_{\mathrm{r}} + \mathbf{\hat C}(\mathbf{q},\mathbf{\dot q}){\mathbf{{\dot q}}_\mathrm{r}} + \mathbf{\hat G}(\mathbf{q})} + \mathbf{\hat{d}} + \mathbf{Hs}
\label{control_design}
\end{equation}
where $\hat{\mathbf{d}}$ is the error compensation term for dynamics estimation defined in equation (\ref{eq_hat_d}) and $\mathbf{\hat D}(\mathbf{q})$, $\mathbf{\hat C}(\mathbf{q},\mathbf{\dot q})$, $\mathbf{\hat G}(\mathbf{q})$ are the learned dynamics model as Equations (\ref{L_D})-(\ref{L_G}), then $\mathbf{q}\rightarrow\mathbf{q}_\mathrm{d}$ as $t\rightarrow\infty$.
\end{proposition}
\begin{proof}
(i) The derivative of the Lyapunov function (\ref{eq_Lya}) is described by
\begin{align}
\dot V(t) &= {\mathbf{s}^\top}\mathbf{\hat D}(\mathbf{q})\mathbf{\dot s} + \frac{1}{2}{\mathbf{s}^\top}\mathbf{\dot{\hat{D}}}(\mathbf{q})\mathbf{s} + {\mathbf{z}^\top}\mathbf{A}\mathbf{\dot z} \notag \\
{\rm{       }} &= {\mathbf{s}^\top}\mathbf{\hat D}({{\mathbf{\ddot q}}_\mathrm{r}} - \mathbf{\ddot q}) + \frac{1}{2}{\mathbf{s}^\top}\mathbf{\dot{\hat{D}}}(\mathbf{q})\mathbf{s} + {\mathbf{z}^\top}\mathbf{A}\mathbf{\dot z}.
\label{e31}
\end{align}
(ii) By substituting the dynamics of the system (\ref{eq_learn_dyn}) into Equation (\ref{e31}) and considering that $\mathbf{\dot {\hat{D}}}(\mathbf{q}) - 2\mathbf{\hat{C}}(\mathbf{q},\mathbf{\dot q})$ is a skew-symmetric matrix, $\dot V(t)$ is simplified to:
\begin{equation}
\dot{V}={{\mathbf{s}}^{\top}}(\mathbf{\hat{D}}(\mathbf{q}){{\mathbf{\ddot{q}}}_\mathrm{r}}+\mathbf{\hat{C}}(\mathbf{q},\mathbf{\dot q}){{\mathbf{\dot{q}}_\mathrm{r}}}+\mathbf{\hat{G}}(\mathbf{q})+\mathbf{d}-\mathbf{u})+{\mathbf{z}^\top}\mathbf{A}\mathbf{\dot z}.
\label{e32}
\end{equation}
(iii)The control $\mathbf{u}$ is designed as Equation (\ref{control_design}). By substituting the control $\mathbf{u}$ into Equation (\ref{e32}) and simplifying, we can get
\begin{equation}
\dot V(t) =  - {\mathbf{s}^\top}\mathbf{H} \mathbf{s} - {\mathbf{s}^\top}\mathbf{z} + {\mathbf{z}^\top}\mathbf{A}{\mathbf{\dot z}}.
\label{e33}
\end{equation}
Due to the supposition $\dot{\mathbf{d}}=0$ in \textit{Proposition 2}, $\dot{\mathbf{z}} = \dot{\hat{\mathbf{d}}} - \dot{\mathbf{d}} = \dot{\hat{\mathbf{d}}}= {{\bf{A}}^{ - 1}}{\bf{s}}$ and
then equation (\ref{e33}) is  
\begin{align}
\dot V(t) &= - {\mathbf{s}^\top}\mathbf{H} \mathbf{s} - {\mathbf{s}^\top}\mathbf{z} + {\mathbf{z}^\top}\mathbf{A}{\mathbf{\dot z}} \notag \\
&=- {\mathbf{s}^\top}\mathbf{H} \mathbf{s} - {\mathbf{s}^\top}\mathbf{z} + {\mathbf{z}^\top}\mathbf{A}\mathbf{A}^{{\mathrm{-1}}}\mathbf{s} \notag \\
&=- {\mathbf{s}^\top}\mathbf{H} \mathbf{s}<0, \forall \mathbf{s} \ne \mathbf{0}.
\label{e34}
\end{align}
Based on the negative definiteness of $\dot{V}$ and the invariance principle in \cite{khalil2002nonlinear}, it follows that $\mathbf{q}\rightarrow\mathbf{q}_\mathrm{d}$ as $t\rightarrow\infty$ under the control (\ref{control_design}).
\end{proof}

\begin{remark}
In practical implementations, directly employing the nominal parameter adaptation law $\dot{\mathbf{z}} = \mathbf{A}^{-1}\mathbf{s}$ may trigger significant transient oscillations or even jeopardize system stability. This susceptibility primarily arises from substantial estimation inaccuracies during the nascent stages of training or the influence of unknown bounded perturbations. To bolster the system's robustness throughout the initial training phase and ensure global boundedness, the adaptation law is augmented with a leakage term featuring a time-varying relaxation factor. Consequently, the modified adaptation law is formulated as:
\begin{equation}
\mathbf{\dot z} = -\alpha(t) \mathbf{z} + \mathbf{A}^{-1}\mathbf{s}, \quad \alpha(t)>0.
\label{eq_leakage_update}
\end{equation}
Substituting equation (\ref{eq_leakage_update}) into equation (\ref{e33}), we can obtain
\begin{align}
\dot V(t)  
 = - {\mathbf{s}^\top}\mathbf{H} \mathbf{s} -\alpha(t) {\mathbf{z}^\top}\mathbf{A}\mathbf{z} . 
\end{align}
The relaxation gain $\alpha(t)$ is designed as a monotonically decreasing function. During the early training phase, the adaptation law (\ref{eq_leakage_update}) is utilized to facilitate the simultaneous convergence of the sliding variable $\mathbf{s}$ and the compensation error $\mathbf{z}$. As the fidelity of the dynamics estimation improves and the modeling error $\mathbf{d}$ attenuates, thereby approximately satisfying the quasi-static condition $\dot{\mathbf{d}} \approx 0$, the time derivative of the Lyapunov function asymptotically recovers its nominal form as expressed in (\ref{e34}).

\end{remark}

\subsection{Implementation}
For clarity and ease of implementation, the complete procedure of the proposed algorithm is described in the form of pseudocode below. The pseudocode outlines each step of the algorithm in a structured manner, making the workflow explicit and reproducible.

\begin{algorithm}[H]
\caption{L-Learning}
\label{alg:L-Learning}
\begin{algorithmic}[1]
\Statex \textbf{Input:} 1. An initialized controller. 2. A desired trajectory $\mathbf{q}_\mathrm{d}(t)$.
\Statex \textbf{Output:} Tracking controller $\mathbf{u}=\mathbf{c}(\mathbf{q},\mathbf{\dot q})$.
\State \textbf{Initialization:} Set the controller (\ref{control_design}) with initial $\mathbf{\hat{D}}_0,\mathbf{\hat{C}}_0,\mathbf{\hat{G}}_0$, initialize the Replay Buffer \(\mathcal{M}\), the standard deviation \(I_0\), and the iteration termination indicator $E_{min}$

\ForAll{\(k= 1,\dots,K\) }
\State  Reduce the standard deviation of noise $I_k=I_0(1-\frac{k}{K})$
\State  Update the noise $\mathrm{Rand}(0,I_k)$
\State  Add noise to the controller $\mathbf{c}'_{k-1}=\mathbf{c}_{k-1}+\mathrm{Rand}(0,I_k)$
\State  Get $\mathcal{D}_k={({\mathbf{q}_{{k_i}}},\mathbf{\dot q}_{k_i},\mathbf{\ddot q}_{k_i},\boldsymbol{\tau}_{k_i})}$ with $\mathbf{c}'_{k-1}(\mathbf{q}_{k_i},\mathbf{\dot q}_{k_i})$
\State  $\mathcal{M}\leftarrow\mathcal{M}\cup\mathcal{D}_k$
\State  Randomly sample batches $\mathcal{B}_k$ from $\mathcal{M}$
\For{each batch in $\mathcal{B}_k$}
\State  Update parameters $\vartheta_{i+1}, \xi_{i+1} \leftarrow \vartheta_{i}, \xi_{i} - \eta \nabla_{\vartheta, \xi} \hat{L}$
\EndFor
\State $\mathbf{\hat{D}}_k,\mathbf{\hat{C}}_k,\mathbf{\hat{G}}_k\leftarrow\hat{L}_k$ as Equation (\ref{L_D})-(\ref{L_G})
\State Generate $\mathbf{c}_k(\mathbf{q},\mathbf{\dot q})$ using $
\mathbf{\hat{D}}_k,\mathbf{\hat{C}}_k,\mathbf{\hat{G}}_k$ as Equation (\ref{control_design})
\State Testing the tracking performance of $c_k(\mathbf{q},\mathbf{\dot q})$ and obtain the corresponding $E_k$
\State \textbf{if} $E_k<E_{min}$ \textbf{do} Select $\mathbf{c}_k(\mathbf{q},\mathbf{\dot q})$ as the final controller $\mathbf{c}(\mathbf{q},\mathbf{\dot q})$ and then break out
\State \textbf{else} Continue
\EndFor
\end{algorithmic}
\end{algorithm}

\begin{remark}
During the dynamics model learning process, a high level of noise is initially injected into the driving force and then gradually reduced as training progresses. Introducing substantial noise in the early stage encourages diverse system responses, thereby expanding the sampled state space. This mechanism is analogous to maximizing policy entropy \cite{sutton1998reinforcement,haarnoja2018soft,moerland2023model}, and it facilitates the acquisition of a more comprehensive characterization of the system dynamics.
\end{remark}

\section{EXPERIMENT}

To verify the learning efficiency and control performance of the proposed L-Learning algorithm, we deployed it in the simulation systems of a 2-degree-of-freedom (2-DOF) robotic arm and a quadrotor UAV to track the pre-set desired trajectories. Meanwhile, we conducted a comprehensive comparison between the proposed method and the classical control method PID, as well as the reinforcement learning algorithms SAC \cite{haarnoja2018soft} and TD3 \cite{fujimoto2018addressing} (both of which are baseline methods provided by Stable-Baselines3 \cite{raffin2021stable}). Regarding the parameter selection problem of PID in the experiment, we first determine the search range based on experience, and then use the Differential Evolution (DE) algorithm to search for a set of control parameters that optimize the tracking performance index. In SAC and TD3, the observation space consists of system kinematics, trajectory tracking errors, and time. The action space is defined as joint torques for the 2-DOF arm and rotor speed deviations from the hover point for the UAV. Crucially, to incentivize high-precision tracking, the reward function adopts a Gaussian-like formulation, defined as the weighted sum of the exponentials of the negative squared tracking errors (specifically covering attitude, angular velocity, and altitude deviations). This bounded reward design facilitates stable convergence during training. The physical simulator for the 2-DOF robotic arm is built directly on the basis of its dynamic equations. For the UAV simulation experiment, we selected the Crazyflie 2.0 quadrotor as the research object and conducted tests in the PyBullet physics engine using the open-source gym-pybullet-drones framework \cite{panerati2021learning}. This framework is a quadrotor simulation platform specifically designed for reinforcement learning and control research. Among them, PyBullet provides underlying rigid body dynamics, including six-degree-of-freedom motion, aerodynamic drag, and rotor thrust-torque coupling, which can realistically simulate the dynamics of UAVs. The model parameters of the 2-DOF robotic arm and the Crazyflie 2.0 quadrotor UAV are presented in Table \ref{tab:2darm_params} and Table \ref{tab:crazyflie_params}, respectively. To facilitate the presentation of results, we define \textbf{RMSE}: ${(\tfrac{1}{T}\int_{0}^{T} e(t)^2 \mathrm{d}t)}^{\frac{1}{2}}$, and
\textbf{ITAE}: $\int_{0}^{T} t \cdot |e(t)| \mathrm{d}t$, where $e \in \mathbb{R}$ is the scalar tracking error.
\begin{table}[htbp]
\centering
\caption{2-DOF Robotic Arm Model Parameters}
\label{tab:2darm_params}
\begin{tabular}{llll}
\toprule
\textbf{Parameter} & \textbf{Symbol} & \textbf{Value} & \textbf{Unit} \\
\midrule
Mass of Arm 1       & $m_1$           & \num{1.0}      & \si{kg} \\
Mass of Arm 2       & $m_2$           & \num{1.0}      & \si{kg} \\
Length of Arm 1     & $l_1$           & \num{1.0}      & \si{m} \\
Length of Arm 2     & $l_2$           & \num{1.0}      & \si{m} \\
Moment of Inertia (Arm 1\& 2)   & $\mathbf{J}$  & \num{0.083}      & \si{kg.m^2} \\
Max torque per Joint & $\tau_{\max}$    & \num{30.0}       & \si{N.m} \\
Simulation step      & $\Delta t_s$  & $1/240$          & \si{s} \\
Control step         & $\Delta t_c$  & $1/48$           & \si{s} \\
\bottomrule
\end{tabular}
\end{table}

\begin{table}[htbp]
\centering
\caption{Crazyflie 2.0 Model Parameters}
\label{tab:crazyflie_params}
\begin{tabular}{llll}
\toprule
\textbf{Parameter} & \textbf{Symbol} & \textbf{Value} & \textbf{Unit} \\
\midrule
Mass                 & $m$           & \num{0.027}      & \si{kg} \\
Arm length           & $l$           & \num{0.046}      & \si{m} \\
Rotor radius         & $r$           & \num{0.022}      & \si{m} \\
Inertia (x,y,z)      & $\mathbf{J}$  & $10^{-5}\mathrm{diag}(1.4,\,1.4,\,2.2)$ & \si{kg.m^2} \\
Max thrust per rotor & $T_{\max}$    & \num{0.16}       & \si{N} \\
Thrust coefficient   & $k_f$         & \num{3.16e-10}   & \si{N.s^2} \\
Torque coefficient     & $k_m$         & \num{7.94e-12}   & \si{N.m.s^2} \\
Simulation step      & $\Delta t_s$  & $1/240$          & \si{s} \\
Control step         & $\Delta t_c$  & $1/48$           & \si{s} \\
\bottomrule
\end{tabular}
\end{table}

\subsection{2-DOF robotic arm}
In this section, we verified the proposed method by performing trajectory tracking tasks on a 2-DOF robotic arm and compared it with PID, SAC, and TD3 algorithms. This robotic arm system consisted of two degrees of freedom: the angle $\alpha$ between the first link and the vertical line in the base coordinate system, and the angle $\beta$ of the second link relative to the first link, with counterclockwise rotation defined as positive.The state of the system was represented by $(\alpha, \beta, \dot{\alpha}, \dot{\beta})$. For this system, we set two different desired trajectories, denoted by $(\alpha_\mathrm{d}, \beta_\mathrm{d})$. As shown in Figure \ref{fig:2Drobotarm_track}, the desired trajectory in Figure(A) required both links of the system to perform periodic sinusoidal motions, while the desired trajectory in Figure(B) required both links to stably maintain the specified state in the form of a quintic polynomial. The tracking error of the 2-DOF robotic arm was defined as: $e={\lVert {(\alpha ,\beta ) - ({\alpha _\mathrm{d}},{\beta _\mathrm{d}})} \rVert_{2}}$.

\begin{figure}
    \vspace{8pt}
    \centering
    \includegraphics[width=0.98\linewidth]{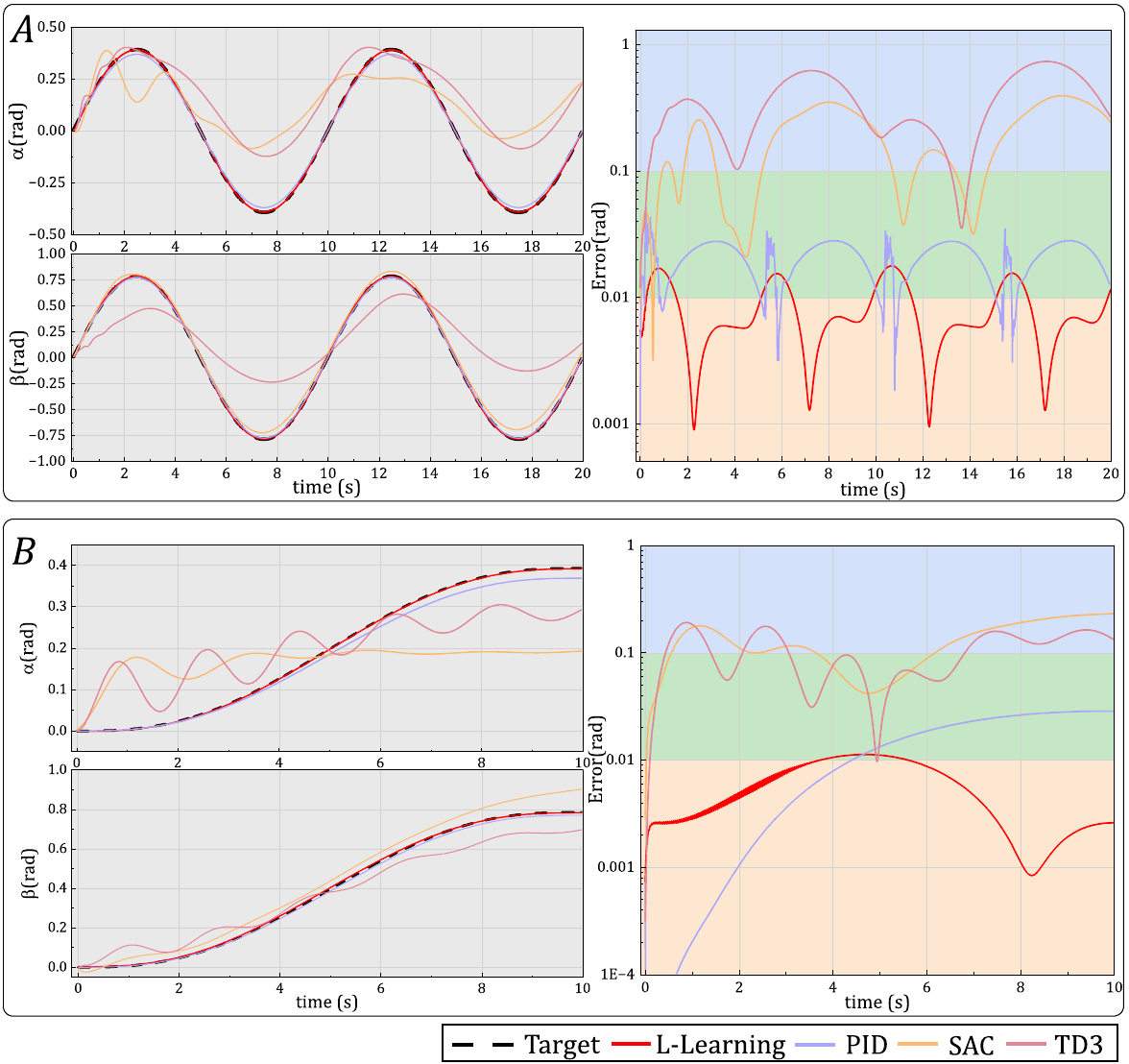}
    \caption{\textrm{Tracking performance of the 2-DOF robotic arm under different algorithms}. In the legend, “Target” denotes the desired trajectory. Results of L-Learning, SAC, and TD3 are obtained with 10,000 training samples.}
    \label{fig:2Drobotarm_track}
\end{figure}

\begin{table}[htbp]
\centering
\caption{Performance Metrics For 2-DOF Robotic Arm}
\label{tab:Metrics_2darm}
\begin{tabular}{llllll}
\toprule
\textbf{Method} & \textbf{Trajectory} & \textbf{Samples} & \textbf{RMSE} & \textbf{ITAE} & \textbf{Time}\\
\midrule
\textbf{Ours}     &  Sine   & \num{10000} & \num{0.009}  & \num{1.561} & \num{01}min \num{10}s\\
PID      &  Sine   & - & \num{0.022}  & \num{4.152} & -\\
SAC      &  Sine   & \num{10000} & \num{0.225}  & \num{45.047} & \num{1}min \num{55}s\\
SAC      &  Sine   & \num{50000} & \num{0.021}  & \num{3.919} & \num{10}min \num{49}s\\
TD3      &  Sine   & \num{10000} & \num{0.412}  & \num{81.364} &\num{01}min \num{13}s\\
TD3      &  Sine   & \num{50000} & \num{0.029}  & \num{5.409} &\num{09}min \num{30}s\\
\midrule
\textbf{Ours}     &  Quintic   & \num{10000} & \num{0.007}  & \num{ 0.259} & \num{01}min \num{10}s\\
PID      &  Quintic   & - & \num{0.018}  & \num{1.013} & -\\
SAC      &  Quintic   & \num{10000} & \num{0.142}  & \num{7.492} & \num{02}min \num{54}s\\
SAC      &  Quintic   & \num{30000} & \num{0.023}  & \num{1.202} & \num{07}min \num{9}s\\
TD3      &  Quintic   & \num{10000} & \num{0.114}  & \num{5.557} &\num{01}min \num{17}s\\
TD3      &  Quintic   & \num{30000} & \num{0.035}  & \num{1.575} &\num{04}min \num{3}s\\

\bottomrule
\end{tabular}

\begin{tablenotes}
\footnotesize
\item \textbf{Note:} ``Sine" means the desired trajectory in sinusoidal form, and ``Quintic" means the desired trajectory in the form of quintic polynomial. ``\textbf{Time}" includes sample collection time and model training time.
\end{tablenotes}
\end{table}

\begin{figure*}[thbp]
\vspace{10pt}
\centering
\subfloat{\includegraphics[width=6.95in]{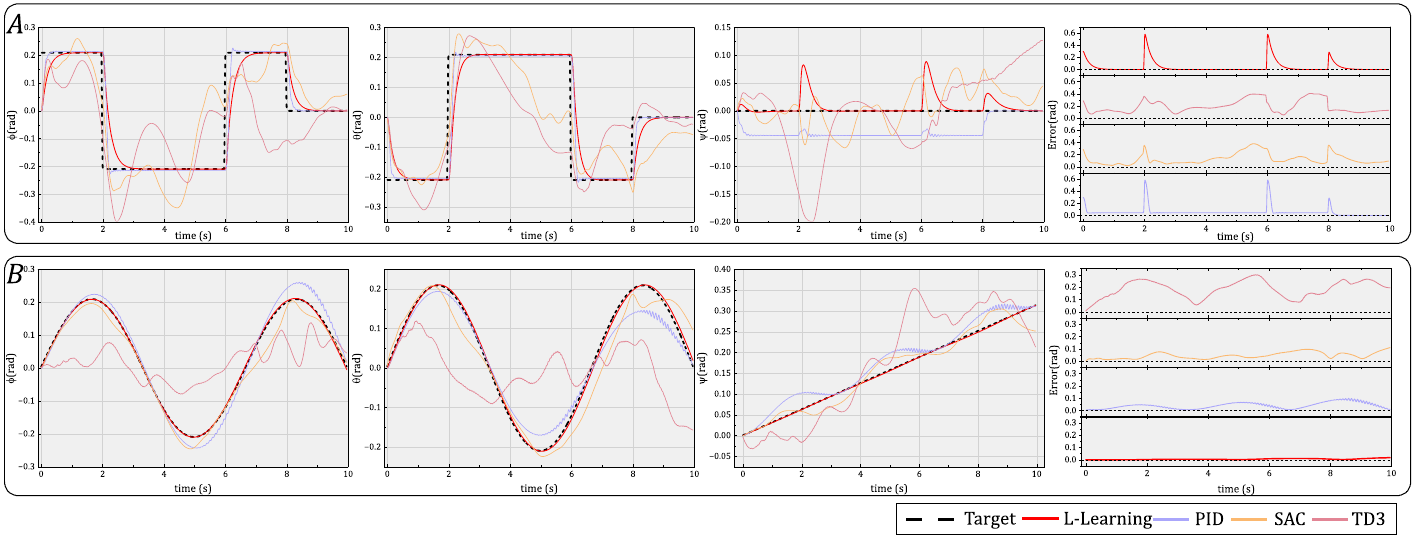}} \label{X}
\caption{\textrm{Attitude tracking control performance of the Crazyflie 2.0 quadrotor UAV under different algorithms}. In the legend, “Target” denotes the desired trajectory. The results of L-Learning, SAC, and TD3 are obtained with 100,000 training samples.} 
\label{figure:uav_tracking}
\end{figure*}

The experimental results of the trajectory tracking for the 2-DOF robotic arm under different algorithms were shown in Figure \ref{fig:2Drobotarm_track} and Table \ref{tab:Metrics_2darm}. For both desired trajectories, L-Learning exhibited superior tracking accuracy, indicating that it could effectively capture the dynamic characteristics of the system and achieve high-precision control.Compared with traditional PID control, L-Learning incorporated data-driven system dynamics information into the control law design, which enabled it to more fully capture the system's dynamic characteristics and achieve a faster tracking response.By comparing reinforcement learning methods such as SAC and TD3, it was found that under the same amount of sample data, the controller obtained by L-Learning achieved better tracking performance, while SAC and TD3 required more samples to approximate its performance (as shown in Table \ref{tab:Metrics_2darm}). This result verified the superiority of L-Learning in terms of data utilization efficiency and learning efficiency.


\subsection{Crazyflie 2.0: The Quadrotor UAV}
In this section, we validated the proposed method by conducting attitude trajectory tracking tasks on a Crazyflie 2.0 quadrotor UAV, providing a comprehensive comparison against PID, SAC, and TD3 algorithms. To faithfully reproduce the dynamic characteristics of the physical platform, we constructed the simulation environment within PyBullet, utilizing URDF files to accurately model the structural and inertial parameters of the quadrotor. Unlike the planar robotic arm where joint angles were naturally chosen, for the quadrotor attitude dynamics, we defined the generalized coordinates as the Euler angles, i.e., $\mathbf{q} = \boldsymbol{\chi} = [\phi, \theta, \psi]^\top$, where $\phi$, $\theta$, and $\psi$ represented the roll, pitch, and yaw angles, respectively. Consequently, the corresponding non-conservative generalized forces were defined as the control torques acting around the body axes, denoted as $\mathbf{u} = \boldsymbol{\tau} = [\tau_\phi, \tau_\theta, \tau_\psi]^\top$. Since the primary objective of this experiment was attitude regulation, we established a baseline rotor speed (RPM) derived from the UAV's mass and lift coefficient. This provided sufficient lift to maintain a stable hover, thereby isolating the evaluation of the attitude controller's performance. For the trajectory tracking tasks, we designed two distinct reference trajectories, denoted by $\mathbf{q}_\mathrm{d} = [\phi_\mathrm{d}, \theta_\mathrm{d}, \psi_\mathrm{d}]^\top$. The first scenario maintained a constant yaw angle while the pitch and roll angles underwent step changes. The second scenario imposed a linear increase in the yaw angle, during which the pitch and roll angles varied sinusoidally, as was illustrated in Figure \ref{figure:uav_tracking} (A and B). The tracking performance is quantified by the Euclidean norm of the attitude error, defined as $e = \| \boldsymbol{\chi} - \boldsymbol{\chi}_\mathrm{d} \|_2$.

\begin{table}[htbp]
\centering
\caption{Performance Metrics For Crazyflie 2.0 UAV}
\label{tab:Metrics_UAV}
\begin{tabular}{llllll}
\toprule
\textbf{Method} & \textbf{Trajectory} & \textbf{Samples} & \textbf{RMSE} & \textbf{ITAE}  & \textbf{Time}\\
\midrule
\textbf{Ours}     &  Step   & \num{100000} & \num{0.107}  & \num{1.705} & \num{5}min \num{11}s\\
PID      &  Step   & - & \num{ 0.098}  & \num{2.195} & -\\
SAC      &  Step   & \num{100000} & \num{0.156}  & \num{ 6.914} &\num{22}min \num{6}s\\
SAC      &  Step   & \num{250000} & \num{0.090}  & \num{3.255} &\num{51}min \num{28}s\\
TD3      &  Step   & \num{100000} & \num{0.218}  & \num{ 9.955} &\num{14}min \num{33}s\\
TD3      &  Step   & \num{250000} & \num{0.101}  & \num{ 4.236} &\num{46}min \num{17}s\\
\midrule
\textbf{Ours}     &  Sine   & \num{100000} & \num{0.007}  & \num{ 0.392} &\num{5}min \num{11}s\\
PID      &  Sine   & - & \num{0.048}  & \num{2.397} & -\\
SAC      &  Sine   & \num{100000} & \num{0.054}  & \num{2.852} &\num{22}min \num{49 }s\\
SAC      &  Sine   & \num{250000} & \num{0.023}  & \num{1.258} &\num{51}min \num{28}s\\
TD3      &  Sine   & \num{100000} & \num{0.196}  & \num{9.816} &\num{14}min \num{6}s\\
TD3      &  Sine   & \num{250000} & \num{0.172}  & \num{7.321} &\num{38}min \num{33}s\\
TD3      &  Sine   & \num{500000} & \num{0.057}  & \num{2.548} &\num{65}min \num{27}s\\

\bottomrule
\end{tabular}

\begin{tablenotes}
\footnotesize
\item \textbf{Note:} ``Step" refers to the desired trajectory with step change, and ``Sine" means the desired trajectory in sinusoidal form. ``\textbf{Time}" refers to the sum of sample collection time and model training time.
\end{tablenotes}
\end{table}

As was illustrated in Figure \ref{figure:uav_tracking}, the proposed method achieved accurate and stable attitude tracking with a fast transient response across diverse reference trajectories. Compared with PID control, L-Learning better captured and compensated for system dynamics, which resulted in substantially reduced tracking errors. While many reinforcement learning (RL) methods, such as SAC and TD3, offered theoretical optimality guarantees at the cost of extensive data and long training times, the proposed approach did not explicitly pursue strict optimality. Instead, L-Learning aimed to obtain a stable and effective control policy with significantly reduced sample requirements and training duration. As was shown in Table \ref{tab:Metrics_UAV}, L-Learning efficiently exploited dynamic information from limited data. This performance was particularly advantageous for UAV applications, where data collection was costly and rapid deployment was essential.


\section{Conclusion}

This work introduces the L-Learning method, which achieves substantial progress in both theoretical analysis and experimental validation. The method provides an efficient, stable, and data-driven solution for trajectory tracking control in general robotic systems. In addition to its theoretical contributions, L-Learning demonstrates strong empirical performance and broad practical applicability. Looking ahead, our objective is to extend its applicability to more challenging and realistic robotic tasks, particularly those involving environmental uncertainty and sim-to-real transfer. To further strengthen the framework, we plan to enhance the robustness and generalization capability of state representations, especially under noisy or partially observed conditions. Moreover, we aim to explore more effective strategies for leveraging available data to improve training efficiency and control performance. Finally, the unified use of neural networks to model physical dynamics, Lyapunov functions, and controllers opens promising opportunities for end-to-end learning and optimization. We will continue to pursue these directions in future research.

\addtolength{\textheight}{-12cm}   









\bibliographystyle{IEEEtran}
\bibliography{IEEEabrv,refs}

\end{document}